\documentclass[conference, final]{IEEEtran}
\IEEEoverridecommandlockouts
\usepackage{cite}
\usepackage{amsmath,amssymb,amsfonts}
\usepackage{amsthm}
\usepackage{algorithmic}
\usepackage{graphicx}
\usepackage{textcomp}
\usepackage{enumerate}   
\usepackage{xcolor}
\usepackage{booktabs}
\usepackage{diagbox} 
\usepackage[hyphens]{url}  % DO NOT CHANGE THIS
\usepackage[ruled]{algorithm2e}
\usepackage[export]{adjustbox}
\usepackage{multirow}
\usepackage{newfloat}
\usepackage{listings}
\usepackage{pifont}
\def\BibTeX{{\rm B\kern-.05em{\sc i\kern-.025em b}\kern-.08em
    T\kern-.1667em\lower.7ex\hbox{E}\kern-.125emX}}
\begin{document}

\newcommand{\cmark}{\ding{51}}%
\newcommand{\xmark}{\ding{55}}%
\newtheorem{lemma}{Lemma}

\title{On the Volatility of Shapley-Based Contribution Metrics in Federated Learning
\thanks{ This paper was funded in part by the Luxembourg National Research Fund (FNR) under grant number 18047633.}
}

\author{\IEEEauthorblockN{1\textsuperscript{st} Arno Geimer}
\IEEEauthorblockA{\textit{SnT, University of Luxembourg} \\
arno.geimer@uni.lu}
\and
\IEEEauthorblockN{2\textsuperscript{nd} Beltran Fiz}
\IEEEauthorblockA{\textit{SnT, University of Luxembourg} \\
beltran.fiz@uni.lu}
\and
\IEEEauthorblockN{3\textsuperscript{rd} Radu State}
\IEEEauthorblockA{\textit{SnT, University of Luxembourg} \\
radu.state@uni.lu}
}

\maketitle

\begin{abstract}

Federated learning (FL) is a collaborative and privacy-preserving Machine Learning paradigm, allowing the development of robust models without the need to centralize sensitive data. A critical challenge in FL lies in fairly and accurately allocating contributions from diverse participants. Inaccurate allocation can undermine trust, lead to unfair compensation, and thus participants may lack the incentive to join or actively contribute to the federation. Various remuneration strategies have been proposed to date, including auction-based approaches and Shapley-value-based methods, the latter offering a means to quantify the contribution of each participant. However, little to no work has studied the stability of these contribution evaluation methods. In this paper, we evaluate participant contributions in federated learning using gradient-based model reconstruction techniques with Shapley values and compare the round-based contributions to a classic data contribution measurement scheme. We provide an extensive analysis of the discrepancies of Shapley values across a set of aggregation strategies and examine them on an overall and a per-client level. We show that, between different aggregation techniques, Shapley values lead to unstable reward allocations among participants. Our analysis spans various data heterogeneity distributions, including independent and identically distributed (IID) and non-IID scenarios.
% Reckon need to define IID/non-IID? 

\end{abstract}


\section{Introduction}
\label{sec:intro}

Federated Learning (FL), a collaborative method for training Machine Learning models without the entities involved having to share their datasets, has gained a lot of traction in recent years as a prevalent form of collaborative learning \cite{49232}. The growing focus on privacy preservation, driven by the enactment of regulations such as the General Data Protection Regulation (GDPR), emphasizes the importance of FL in addressing data privacy concerns.

It has seen adoption in a wide range of domains, from financial corporations and medical institutions to IoT devices, with major industry players such as Nvidia offering in-house platforms and solutions \cite{roth2022nvidia}.
An important aspect of a federation is the strategy, defining the approach in which the central server computes model updates. Since the inception of Federated Learning, numerous strategies have been introduced to manage specific problems, such as poor performance on heterogeneous datasets \cite{ye2023heterogeneous}, attacks by malicious clients \cite{li2020learning}, and the convergence speed of the global model \cite{reddi2021adaptive}. 

In addition to the performance of the model, the evaluation and fair distribution of rewards to participants in a federation is of utmost importance, particularly in commercial settings where financial incentives might be needed for the participation of the client \cite{zeng2021comprehensivesurveyincentivemechanism}. Each participant would expect to be fairly compensated for their participation, regardless of the underlying aggregation strategy used.

In this work, we use cumulative round-based Shapley values as percentages to represent the contribution of each client. Shapley values, a concept from cooperative game theory, are employed to quantify the contribution of each participant in a federation by evaluating their marginal impact on the model's performance in each communication round. This approach is a popular contribution evaluation method \cite{liu2022gtg} and reflects a practical industry use case where, in each round, the central server assesses the contribution of each client based on their local model updates. Over time, the server aggregates these round-based contributions to determine the overall contribution of each participant, allowing a fair reward distribution. We analyze, across an array of popular aggregation strategies, the overall, as well as per-client, reliability of these generated contribution values.

Focusing on image recognition tasks, our results treat a vast and diverse set of possible Federated Learning scenarios. In addition to using multiple vision datasets popular in FL literature, we simulate data distributions among participants using splits generated using a Dirichlet distribution to control the degree of heterogeneity. We show that, across most tested strategies, round-based Shapley values perform similarly in terms of closeness to the ground truth.

Finally, we study the stability of contributions on a per-client basis and demonstrate that round-based Shapley values cause instability on client-level contributions, rendering them potentially unfit for real-world usage.

% \textbf{} `
The contributions of this paper can be summarized as follows:
\begin{enumerate}[(i)]

\item We analyze the properties and effects of the weight assignment parameter in round-based contribution calculation.
\item We perform extensive experiments studying Federated Learning contribution allocation across a broad range of popular aggregation strategies, providing a novel comprehensive comparative analysis.

\item We discover an inherent instability of Shapley-based contribution evaluation methods across popular aggregation strategies and examine this instability in a cross-silo federation.
\end{enumerate}

\section{Background}
\label{sec:background}
\subsection{Federated Learning}\label{FL}

A basic Federated Learning architecture consists of a group of participants, or clients, who each own a private local dataset, along with a central server or aggregator. Before training starts, they agree on the model architecture and hyperparameters such as learning rates, the aggregation strategy, and local epoch numbers. 

Subsequently, the central server starts the federation by initializing and distributing the first global model. At the onset of each round, the aggregator selects a subset of clients and transmits them the current global model. After receiving the current global model, users train their respective data for the $e$ epochs and upload their new model to the server. The number of local epochs, $e$, has a serious impact on the federation \cite{mcmahan2017communication}: A low value requires more rounds of communication until convergence is achieved, which can be problematic when participants are low-resource IoT devices. However, choosing a value that is too high negatively impacts the performance of the final model, leading to the need to find a balance between performance and communication overhead.

After receiving all chosen client models, the server uses the agreed-upon strategy as an aggregation function to compute the new global model. Repeating this process until satisfactory performance has been achieved allows the collaborating entities to obtain a model incorporating knowledge of all clients, which has been shown to outperform locally trained models \cite{mcmahan2017communication}.

Due to the variety of applications in Federated Learning, two primary setups have emerged. In a \textbf{cross-device} federation, a large number of distributed clients, usually with small volumes of data, collaborate to build a robust model. On the other hand, \textbf{cross-silo} setups involve few participants with large amounts of data working together to create a global model that incorporates the specific knowledge of each entity. In an environment comprised of profit-driven entities, correctly assessing contributions is of utmost importance, as participants may invest money in infrastructure and compliance in hopes of receiving their equitable part of the payout. We focus on cross-silo federations, as they are a more appropriate use case for contribution calculations.

\subsection{Aggregation strategies}\label{aggs}
In Federated Learning, different aggregation strategies may be employed based on the specific circumstances and objectives of the participants.

\begin{itemize}
\item \textsc{FedAvg} \cite{mcmahan2017communication} is the baseline Federated Learning aggregation technique. \\
\item Federated Averaging with Momentum \cite{hsu2019measuring}, or \textsc{FedAvgM}, improves the poor performance of \textsc{FedAvg} on heterogeneous data by adding momentum when updating model weights.\\
\item \textsc{FedAdagrad, FedAdam, FedYogi}, all three proposed in \cite{reddi2021adaptive}, employ advanced gradient-based optimization methods to improve the convergence speed of the federation. With clients' weight updates considered a pseudo-gradient, they use \textsc{Adagrad, Adam} and \textsc{Yogi}, respectively. \\
\item To reduce potentially adversarial updates, \textsc{FedMedian, FedTrimAvg} \cite{yin2021byzantinerobust} update the global model as the median and trimmed mean of client updates, as opposed to a weighted average in \textsc{FedAvg}. In adversarial settings, this has been shown to guarantee substantial improvements over the baseline. \\
\item Finally, \textsc{Krum} \cite{blanchard2017machine}, similarly to the two previous methods, is a modified aggregation to counter adversarial clients. It excludes client updates that are too far away from the other clients.
\end{itemize}

All proposed strategies share the global objective function \[f(w) = \frac{1}{K}\sum_{k=1}^{K}{{F_k(w)}}.\]
Although our investigation includes aggregation strategies aimed at improving defenses against model attacks, we do not include adversarial participants. 

\subsection{Shapley values}\label{Shapley values}
In game theory, a cooperative game of n persons is a game in which subsets of players, called coalitions, can cooperate to obtain a utility $v$ \cite{parrachino2006cooperative}, which they may distribute among themselves. To determine the exact contribution of each player in the coalition to $v$, Shapley values \cite{shapley1951notes} can be used. Specifically, let $S$ be a coalition of players, then the contribution $\phi$ of player $i \in S$ to $v$ can be determined as 
\begin{equation}\label{eq:shapley}
    \phi_{i}(v) = \sum_{S \subseteq N \setminus {i}}{\frac{|S|!(n - |S| -1)!}{n!} \Big[v\big(S \cup {i}\big) - v\big(S\big)\Big]}
\end{equation}
Shapley values have a range of desirable qualities, among which the property that players who bring no value to a coalition have Shapley value zero. In addition, players who bring the same value also receive the same contribution values. In fact, it has been shown in \cite{jia2019towards} that the Shapley value is the only contribution measure that satisfies these properties. This is particularly important in an FL setting, where adversarial parties might try to unfairly increase their perceived contribution by joining the federation as multiple participants or artificially inflating their data. %\textcolor{red}{SOURCE?}

\section{Related Work}\label{sec:relatedwork}
Several approaches have been devised to evaluate participants' contributions during Federated Learning training.
We give a brief overview of the most common approaches, which are either based on numeric computations or on self-disclosed information and estimations by the central server.

\subsubsection*{Self-reported}
The most direct approach seen in previous work is to have participants perform their own reporting, providing information regarding their local model training process \cite{sarikaya2019motivating,zhang2020hierarchically}. The information provided may encompass aspects such as the quality and quantity of data provided by the participant; costs associated with data collection, processing, and communication in the federation. A limitation of these approaches is the assumption of honesty and the ability of clients to evaluate their own conditions. Reputation-based client selection techniques have been proposed to mitigate this issue \cite{10806597}

\subsubsection*{Influence}
Another approach to quantifying the contribution of participants are influence-based methods. These apply a systematic method for quantifying the impact of individual data points on the global model \cite{koh2017understanding}. Influence can be computed by assessing the disparity in models when trained with and without specific data points. This approach was used by \cite{richardson2019rewarding} to provide incentives for clients to provide high-quality data during the training process. Although the influence metric offers detailed insights, its computation can be challenging due to the expensive nature of retraining models.

\subsubsection*{Auctions}
In diverse areas of Federated Learning, auctions serve as a fundamental economic mechanism to allocate resources such as training data and computational power, setting prices through bidding \cite{tu2021incentive}. Although contribution estimation is an important aspect, other key concepts in auction theory include valuation, utility, and social welfare, with the ultimate goal being to enhance the efficiency and effectiveness of resource allocation in FL environments.

\subsubsection*{Shapley values}

Used in game theory to evaluate the contribution of players to a common task, Shapley values are the most common approach to numerically assess individual clients' contribution to the shared model in FL.

\subsection*{Shapley values in Federated Learning}

Although it has been proposed to calculate Shapley values in FL by evaluating the performance of fully trained federated models on subsets of participants \cite{wang2019measure}, a more common approach, called One-Round Reconstruction, is to use gradient-based model reconstruction \cite{song2019profit} which assesses contributions at each federation round, possibly using Monte Carlo methods \cite{wang2020principled} to speed up the calculation.  We will use this method to calculate per-round contributions of all clients, with final contribution values derived from the sum of each client's Shapley values. A detailed calculation can be found in Algorithm \ref{alg:shapley_algorithm}. One-Round Reconstruction is of particular interest in continuous learning tasks, where payouts may occur periodically instead of just once.

\begin{algorithm}
\DontPrintSemicolon
\caption{One-round reconstruction of Shapley values  with \textsc{FedAvg}; $n_k$ is the Shapley value and dataset size of client $k$. $C$ is the set of all clients, with size $K$.}
\label{alg:shapley_algorithm}
\textbf{Server} executes: \;
\For{\upshape each round $t = 1, 2, ...$}{
    \For{\upshape each client $k \in C$}{
        $\phi_k^t = \sum_{S \subseteq C \setminus {k}}{\frac{|S|!(K - |S| -1)!}{K!}F_k^S}$\;
        \textbf{where} $F_k^S = \Big[F_k(\omega_{S \cup k}) - F_k(\omega_{S})\Big]$ \;
        \textbf{with} $\omega_{S} = \sum_{i \in S}\frac{n_i}{n_S}\omega_i^t$ \;
        \textbf{and} $n_S = \sum_{i \in S} n_i$ \;
    }
}
\end{algorithm}

Note that the calculation of $\omega_S$ depends on the chosen aggregation technique and may not be a simple average as shown above. 
Although the literature extensively discusses various evaluation methods and their respective advantages and disadvantages from complexity and communication standpoints, to the best of our knowledge, no prior research has specifically addressed the stability of contributions.

\section{Methodology}
Having established the motivation behind our work, we present the proposed methodology applied to the experiments.

\subsection{Experimental Setup}\label{sec:expsetup}
An open source repository containing an implementation of Shapley value calculations in FL, used in our experiments, is available at {\small \url{https://github.com/arnogeimer/flwr-contributions}}.

\subsubsection{Training data}
The experiments include image classification on \textsc{CIFAR}-10 and \textsc{CIFAR}-100 \cite{krizhevsky2009learning}, \textsc{MNIST} \cite{deng2012mnist} and Fashion-\textsc{MNIST} \cite{xiao2017fashion}, four of the most common tasks in the Federated Learning literature.

\begin{table}[ht]
    \begin{center}
    \begin{tabular}{cccc}
    \hline
        Dataset & Clients & Train examples & Train steps \\
        \hline
        \textsc{Cifar-10} & 3 & 50,000 & 600 \\
        \textsc{Cifar-100} & 3 & 50,000 & 1000 \\
        \textsc{Mnist} & 5 & 60,000 & 120 \\
        \textsc{FMnist} & 5 & 60,000 & 160 \\
        \hline
    \end{tabular}
    \end{center}
    \caption{Dataset statistics.}
    \label{table:dsstats}
\end{table}

Following \cite{li2022federated}, we use a Dirichlet distribution-based size split with concentration parameter $\alpha$ to control the degree of non-IID data distribution. Specifically, we use $\alpha = 1$ for heavily non-IID data, $\alpha = 10$ for slightly non-IID data, and $\alpha = 100$ for an almost uniform distribution across all datasets. Using these values allows us to effectively study the effect of different levels of data heterogeneity.

\subsubsection{Model and training pipeline}
Since the focus of our work is not on model architecture, we deploy a straightforward convolutional neural network on all tasks, without hyperparameter fine-tuning.
The model consists of two 5x5 convolution layers, 2 dense layers with ReLu activation and 16x5x5x120 and 120x84 units, respectively, and a dense layer with 84x10 units. Image preprocessing for CIFAR-10 and CIFAR-100 consists of random cropping, horizontal flipping, and random change of brightness, contrast, saturation, and hue. 
These models are standard in the Federated Learning literature and their performance is adequate.

\subsubsection{Strategies}
We employ a collection of 8 different aggregation strategies, which have all been widely adopted across their respective application domains. Although aggregation strategies that do not have the same objective function as \textsc{FedAvg} have been proposed \cite{li2019fair}, we do not include any of them in our analysis. Performance metrics of strategies, whether related to model accuracy or defense against attacks, are not the primary focus of this paper. We therefore employ out-of-the-box hyperparameters in all strategies.

\subsubsection{Technical specifications}

The models are built with Pytorch 1.13.1 \cite{NEURIPS2019_9015}, using Flower 1.6 \cite{beutel2020flower} as the Federated Learning framework. All training is executed on Nvidia Tesla V100 16 GB VRAM GPUs.

As we used 4 different datasets and three Dirichlet splitting values, the experiments included 12 different use cases. With 70 seeds and 3 epoch values $e \in (2, 5, 10)$ per use case and a total of 8 aggregation strategies, our experiments encompass more than 20.000 full Federated Learning runs.

\subsection{Contribution calculation}

\subsubsection{Round-based aggregation}
We will call the contribution of the client $k$ the normalized weighted sum of their per-round Shapley values $\phi_k^t$. The total contribution of a participant is thus represented on a percentage scale, allowing better comparisons between runs. Although the numerical values of the Shapley values may vary, the proportions remain comparable. Shapley values are aggregated using an inverse linear factor for rounds, up to a maximum contribution halting round $R$. For example, at $R = 10$, the Shapley values are added using weight $\frac{10}{10}$ in round $1$, $\frac{9}{10}$ in round 2, etc:
\begin{equation}
    \phi_k(R) = \sum_{t = 0}^R{\frac{R-t}{R}\phi_k^t}
\end{equation}
The final contribution of client $k$, as a percentage value, is $\phi_k \times (\sum_{i\in C}{\phi_i})^{-1}$.

\subsubsection{Ground truth}
Determining which aggregation strategy produces the values closest to a ground truth is of the utmost importance. To this end, we establish the ground-truth contribution relative to the sizes of the participants' datasets. Since we do not operate in an adversarial environment, where actors may lie about their dataset statistics, and the distribution of classes across different splittings is nonheterogeneous, we deem the size of splits to be a reliable contribution assessment for clients. This method has previously been used in the literature\cite{guo2024fair}\cite{9906094}.
The correctness of a contribution is calculated using the squared Euclidean distance to the ground truth, in line with previous studies. To better understand our results, we establish a lower bound using an equal payout contribution evaluation approach.

\begin{lemma}\label{thm:lemma}
    Let a data set $D$ be split by size into $n$ different subsets $D_1, ..., D_n$ following a Dirichlet distribution $Dir((\alpha, ..., \alpha))$. Then, an equal payout differs from a size-based payout, on average, by $d = \frac{n-1}{n^2 \alpha + n}$ under the squared Euclidean distance.
\end{lemma}

\begin{proof}
Representing the size-based payout by a Dirichlet-random variable $X \sim Dir((\alpha, ..., \alpha))$, the expected squared Euclidean difference to an equal payout is 
\begin{align*}
    d & =  \mathbb{E} \left[ \left\| \left( X - \left(\frac{1}{n}, ..., \frac{1}{n}\right) \right) \right\|_2^2 \right] \\
    & = \mathbb{E} \left[ \sum_{i=1}^{n}{\left( X_i - \mu(X_i) \right)^2} \right] \\
    & = \sum_{i=1}^{n}{\mathbb{E} \left[ (X_i - \mu(X_i))^2 \right] } \\
    & = \sum_{i=1}^{n}{\text{var}(X_i)} \\
    & = n \times \frac{n-1}{n^3\alpha+n^2} \\
\end{align*}
\end{proof}

The lemma gives an upper bound on the performance of any contribution method, provided that size is the ground truth. Given any Dirichlet-based data split, if a contribution method's result does not, on average, lie within $\frac{n-1}{n^2 \alpha + n}$ of the ground truth, it is more sensible to use an equal payout instead of that contribution method.

\subsection{Optimizing the Contribution Halting Round $R$}
Figure \ref{fig:bestR} contains the density plots of the optimal values for $R$ over all experiments with the respective aggregation strategy.
It shows that determining an optimal value for $R$ is not straightforward: No strategy displays a clear, easy-to-analyze distribution. In all cases, an optimal value appears across all rounds, with a skew towards the beginning, or the end, of the Federated Learning process. We will use the mean optimal $R$ of the respective strategy in our experiments. In applications, a relatively early stopping point for contribution evaluation proves beneficial, as it requires less computational expenses on the server side. However, this will be decided on a case-by-case basis.

\begin{figure}[ht]%{0.4 cm}
   \centering
    \includegraphics[scale = .45]{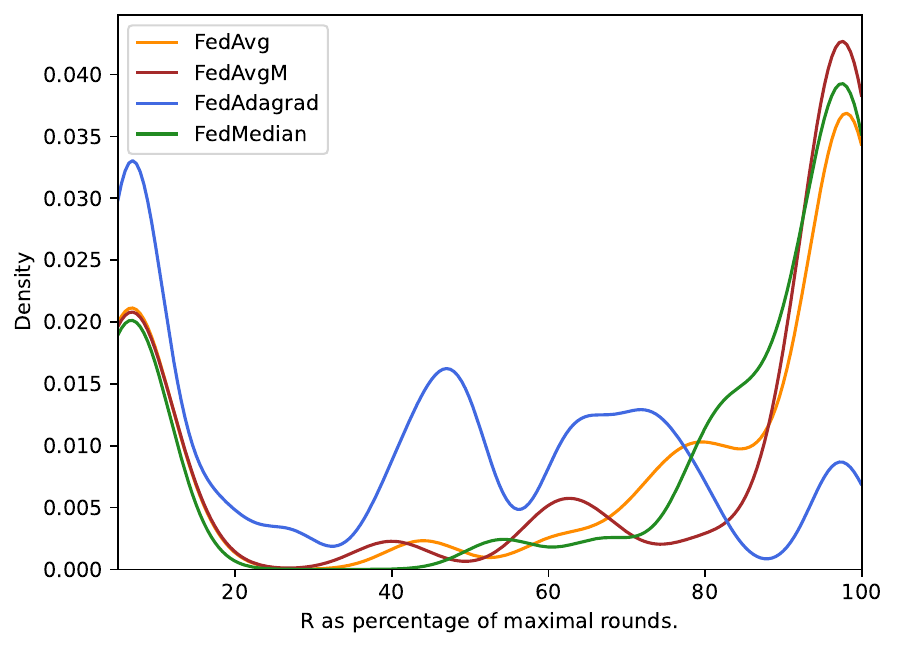}
   \caption{Density plots of the best value for $R$, as percentage of total rounds, for select strategies. An optimal value minimizes the distance to the ground truth in a run.}
   \label{fig:bestR}
\end{figure}

\section{Contribution Performance}

\begin{table*}[ht]
    \begin{center}
    \begin{tabular}{ll|c|c|c| |c|c|c| |c|c|c| |c|c|c}
        \cline{3-14}
        & & \multicolumn{3}{c||}{\textsc{Cifar-10}} & \multicolumn{3}{c||}{\textsc{Cifar-100}} & \multicolumn{3}{c||}{\textsc{Mnist}} & \multicolumn{3}{c}{\textsc{FMnist}} \\
        \hline
        \multicolumn{1}{l||}{$e$} &{\diagbox[height=2\line]{Strategy}{$\alpha$}} & 1 & 10 & 100  & 1 & 10 & 100 & 1 & 10 & 100 & 1 & 10 & 100 \\

        \hline
        \hline
        \multicolumn{1}{l||}{}& \multicolumn{1}{l|}{FedAvg}   & \underline{3.71} & \underline{0.86} & \textbf{0.77} & \underline{\textbf{0.62}} & \underline{0.76} & 0.96 & \underline{3.42} & \underline{1.3} & 1.18 & \underline{3.63} & \underline{1.12} & 0.65\\
        \multicolumn{1}{l||}{}& \multicolumn{1}{l|}{FedAvgM}   & \underline{3.66} & \underline{0.9} & 1.05 & \underline{0.7} & \underline{\textbf{0.73}} & 1.08 & \underline{3.53} & \underline{1.11} & 1.2 & \underline{3.59} & \underline{1.03} & 0.68\\
        \multicolumn{1}{l||}{}& \multicolumn{1}{l|}{FedAdagrad}   & \underline{8.21} & \underline{1.53} & 0.93 & \underline{3.02} & \underline{1.18} & \textbf{0.74} & \underline{4.06} & \underline{1.51} & 1.27 & \underline{4.22} & \underline{1.07} & 0.85\\
        \multicolumn{1}{l||}{2}& \multicolumn{1}{l|}{FedAdam}   & \underline{3.5} & 2.39 & 1.13 & \underline{1.04} & \underline{1.19} & 1.15 & \underline{\textbf{2.43}} & 1.85 & 1.63 & \underline{3.98} & 1.62 & 1.21\\
        \multicolumn{1}{l||}{}& \multicolumn{1}{l|}{FedYogi}   & \underline{\textbf{2.4}} & \underline{1.34} & 0.82 & \underline{1.09} & \underline{1.55} & 1.37 & \underline{4.58} & \underline{1.37} & \textbf{0.5} & \underline{5.77} & \underline{1.51} & 0.64\\
        \multicolumn{1}{l||}{}& \multicolumn{1}{l|}{FedMedian}   & \underline{3.3} & \underline{\textbf{0.79}} & 0.97 & \underline{0.69} & \underline{0.92} & 1.05 & \underline{2.74} & \underline{1.33} & 1.05 & \underline{\textbf{3.49}} & \underline{\textbf{0.95}} & \textbf{0.63}\\
        \multicolumn{1}{l||}{}& \multicolumn{1}{l|}{FedTrimAvg}   & \underline{3.48} & \underline{0.83} & 0.96 & \underline{0.74} & \underline{0.81} & 0.75 & \underline{3.08} & \underline{\textbf{1.1}} & 1.07 & \underline{3.56} & \underline{1.01} & 0.67\\
        \multicolumn{1}{l||}{}& \multicolumn{1}{l|}{Krum}   & 22.71 & 7.14 & 2.3 & 41.23 & 7.76 & 2.79 & \underline{5.93} & 1.85 & 1.23 & \underline{6.38} & 2.35 & 1.05\\
        \hline
        \multicolumn{14}{c}{} \\
        \hline
        \multicolumn{1}{l||}{}& \multicolumn{1}{l|}{FedAvg}   & \underline{3.71} & \underline{0.9} & 1.15 & \underline{0.82} & \underline{\textbf{0.64}} & 0.67 & \underline{3.51} & 1.65 & 1.78 & \underline{3.37} & \underline{1.0} & 0.87\\
        \multicolumn{1}{l||}{}& \multicolumn{1}{l|}{FedAvgM}   & \underline{3.66} & \underline{0.92} & 1.2 & \underline{0.75} & \underline{0.71} & 0.88 & \underline{3.39} & 1.62 & 1.57 & \underline{\textbf{3.23}} & \underline{0.96} & 1.09\\
        \multicolumn{1}{l||}{}& \multicolumn{1}{l|}{FedAdagrad}   & \underline{7.78} & \underline{1.56} & \textbf{0.9} & \underline{3.19} & \underline{0.93} & 1.08 & \underline{4.46} & 1.81 & 1.36 & \underline{3.79} & \underline{\textbf{0.88}} & 1.13\\
        \multicolumn{1}{l||}{5}& \multicolumn{1}{l|}{FedAdam}   & \underline{4.16} & 3.55 & 2.56 & \underline{2.04} & \underline{2.11} & 2.49 & \underline{4.22} & 3.11 & 3.11 & \underline{3.92} & 2.38 & 1.78\\
        \multicolumn{1}{l||}{}& \multicolumn{1}{l|}{FedYogi}   & \underline{\textbf{1.98}} & \underline{1.01} & 0.77 & \underline{\textbf{0.67}} & \underline{1.01} & \textbf{0.5} & \underline{3.48} & \underline{\textbf{0.81}} & \textbf{0.53} & \underline{3.92} & \underline{1.24} & \textbf{0.63}\\
        \multicolumn{1}{l||}{}& \multicolumn{1}{l|}{FedMedian}   & \underline{3.31} & \underline{\textbf{0.73}} & 0.92 & \underline{0.88} & \underline{0.66} & 0.94 & \underline{\textbf{2.96}} & \underline{1.31} & 1.75 & \underline{3.27} & \underline{0.99} & 1.05\\
        \multicolumn{1}{l||}{}& \multicolumn{1}{l|}{FedTrimAvg}   & \underline{3.91} & \underline{0.89} & 0.92 & \underline{0.9} & \underline{0.73} & 0.9 & \underline{3.2} & 1.69 & 1.78 & \underline{3.3} & \underline{0.92} & 0.97\\
        \multicolumn{1}{l||}{}& \multicolumn{1}{l|}{Krum}   & 20.14 & 7.35 & 2.54 & 26.03 & 4.9 & 1.84 & \underline{6.09} & 2.37 & 2.25 & \underline{6.19} & 2.48 & 1.61\\
        \hline
        \multicolumn{14}{c}{} \\
        \hline
        \multicolumn{1}{l||}{}& \multicolumn{1}{l|}{FedAvg}   & \underline{4.32} & \underline{\textbf{0.85}} & \textbf{1.17} & \underline{1.12} & \underline{\textbf{0.56}} & 1.1 & \underline{3.63} & 2.24 & 2.2 & \underline{3.14} & \underline{\textbf{1.09}} & 1.35\\
        \multicolumn{1}{l||}{}& \multicolumn{1}{l|}{FedAvgM}   & \underline{4.04} & \underline{0.86} & 1.36 & \underline{1.23} & \underline{0.63} & \textbf{0.8} & \underline{3.43} & 1.89 & 2.39 & \underline{3.31} & \underline{1.12} & 1.58\\
        \multicolumn{1}{l||}{}& \multicolumn{1}{l|}{FedAdagrad}   & \underline{7.82} & \underline{1.42} & 1.53 & \underline{3.49} & \underline{1.45} & 0.95 & \underline{3.43} & \underline{\textbf{1.43}} & \textbf{1.12} & \underline{3.28} & \underline{1.42} & 1.48\\
        \multicolumn{1}{l||}{10} & \multicolumn{1}{l|}{FedAdam}   & \underline{4.55} & 3.15 & 1.64 & \underline{2.68} & \underline{1.12} & 1.4 & \underline{4.71} & 2.99 & 2.95 & \underline{4.16} & 3.31 & 2.85\\
        \multicolumn{1}{l||}{}& \multicolumn{1}{l|}{FedYogi}   & \underline{\textbf{2.28}} & \underline{1.54} & 1.27 & \underline{1.34} & \underline{0.8} & 1.34 & \underline{4.63} & 3.14 & 1.87 & \underline{3.13} & 1.83 & 1.51\\
        \multicolumn{1}{l||}{}& \multicolumn{1}{l|}{FedMedian}   & \underline{3.16} & \underline{1.17} & 1.74 & \underline{1.0} & \underline{0.97} & 1.33 & \underline{\textbf{2.94}} & 2.12 & 2.24 & \underline{3.28} & 1.67 & 1.45\\
        \multicolumn{1}{l||}{}& \multicolumn{1}{l|}{FedTrimAvg}   & \underline{4.1} & \underline{0.93} & 1.27 & \underline{\textbf{0.88}} & \underline{0.64} & 0.94 & \underline{3.11} & 2.02 & 2.31 & \underline{\textbf{3.04}} & \underline{1.11} & \textbf{1.19}\\
        \multicolumn{1}{l||}{}& \multicolumn{1}{l|}{Krum}   & \underline{15.85} & 7.07 & 4.04 & 17.63 & 4.57 & 2.11 & \underline{6.02} & 3.8 & 3.67 & \underline{5.94} & 3.36 & 2.37\\
        \hline
        \multicolumn{14}{c}{} \\
        \cline{2-14}
        & \multicolumn{1}{l|}{Expected equal error} & 16.67 & 2.15 & 0.22 & 16.67 & 2.15 & 0.22 & 13.33 & 1.57 & 0.16 & 13.33 & 1.57 & 0.16 \\
        \cline{2-14}
        \multicolumn{14}{c}{} \\
        
       \end{tabular}
        \caption{Average square Euclidean distance from per-strategy contributions to the size-based ground truth, lower values are better. The last row contains the expected difference to an equal payout, as determined in lemma \ref{thm:lemma}. \underline{Underlined} entries beat the equal allocation, \textbf{bold} entries are the best in the column. All values were multiplied by 100 for readability.}
       \label{table:means}
    \end{center}
\end{table*}

Table \ref{table:means} shows the squared Euclidean distance from the ground truth in all experiments. We note that most strategies outperform an equal payout when dataset sizes are heavily heterogeneous. The same can be observed in the slightly non-iid case. 
In the case of $\alpha = 100$, equal contributions beat all strategies as expected, as Dirichlet random values with such a high concentration parameter have extremely low variance. We can conclude that, for most practical purposes, employing an actual contribution calculation strongly surpasses an equal payout. These results demonstrate the validity of the implementation, and of round-wise Shapley values.

However, we observe performance differences between the aggregation strategies. Although they are not often substantial, being aware that different aggregation methods may not yield the exact same performance in contribution estimation is important when establishing a federation. In particular, we note that no single aggregation strategy clearly outperforms all others. In fact, the best-performing strategy seems to change randomly between experiments. Although \textsc{FedYogi} surpasses the other strategies in eight out of 36 scenarios, this is not enough to draw the conclusion that it is a clearly superior aggregation method when it comes to per-round Shapley evaluation. The only conclusion which can be drawn is that \textsc{Krum} is severely underperforming.

We conclude that all aggregation strategies result in satisfactory contribution evaluations. We suggest that, when proposing new contribution methods for Federated Learning, our results be taken into account by including multiple different aggregation techniques. We have shown that new aggregation mechanisms do not necessarily perform well with any contribution method, as is the case with \textsc{Krum} in our results. This important observation shows that an out-of-the-box strategy, in combination with a contribution mechanism, may lead to poor results even though both the strategy and the contribution method perform adequately on their own.

Overall, the results in Table \ref{table:means} suggest that using per-round Shapley values for contribution evaluation allows the central server to accurately approximate each client's contribution to the shared model and thus reward them fairly. This appears true, as all strategies result in contributions that lie relatively close to the ground truth. 

However, Table \ref{table:abss} shows that a minimal distance from the ground truth does not necessarily mean that all clients are paid fairly. Indeed, we observe that, on average, the clients' rewards are in disagreement by more than $10\%$ in highly heterogeneous settings for even the best-performing strategies. We conclude that solely minimizing the Euclidean distance does not provide good information with regard to the performance of a contribution mechanism, as an overall decent result does not imply that each client is fairly compensated. 

\begin{table*}[ht]
    \begin{center}
    \begin{tabular}{ll|c|c|c| |c|c|c| |c|c|c| |c|c|c}
        \cline{3-14}
        & & \multicolumn{3}{c||}{\textsc{Cifar-10}} & \multicolumn{3}{c||}{\textsc{Cifar-100}} & \multicolumn{3}{c||}{\textsc{Mnist}} & \multicolumn{3}{c}{\textsc{FMnist}} \\
        \hline
        \multicolumn{1}{l||}{$e$} & {\diagbox[height=2\line]{Strategy}{$\alpha$}} & 1 & 10 & 100  & 1 & 10 & 100 & 1 & 10 & 100 & 1 & 10 & 100 \\

        \hline
        \hline
        \multicolumn{1}{l||}{} & \multicolumn{1}{l|}{FedAvg}   & 12.63 & 6.13 & \textbf{5.99} & \textbf{5.31} & 5.89 & 6.88 & 12.62 & 7.65 & 7.5 & 13.06 & 7.13 & 5.71\\
        \multicolumn{1}{l||}{} & \multicolumn{1}{l|}{FedAvgM}   & 12.55 & 6.59 & 7.25 & 5.64 & \textbf{5.65} & 6.9 & 12.82 & \textbf{7.31} & 7.73 & 12.97 & 6.68 & 5.48\\
        \multicolumn{1}{l||}{} & \multicolumn{1}{l|}{FedAdagrad}   & 20.72 & 8.61 & 6.44 & 12.32 & 7.29 & \textbf{5.56} & 13.84 & 8.69 & 7.66 & 14.37 & 6.9 & 6.29\\
        \multicolumn{1}{l||}{2} & \multicolumn{1}{l|}{FedAdam}   & 12.7 & 10.1 & 7.14 & 7.26 & 7.21 & 7.09 & \textbf{10.2} & 9.93 & 9.31 & 13.32 & 8.26 & 7.23\\
        \multicolumn{1}{l||}{} & \multicolumn{1}{l|}{FedYogi}   & \textbf{10.39} & 8.21 & 6.49 & 7.22 & 9.17 & 8.18 & 15.22 & 8.15 & \textbf{4.67} & 16.53 & 8.09 & 5.57\\
        \multicolumn{1}{l||}{} & \multicolumn{1}{l|}{FedMedian}   & 12.78 & 6.18 & 6.73 & 5.76 & 6.45 & 7.35 & 11.43 & 7.86 & 7.07 & \textbf{12.71} & \textbf{6.62} & \textbf{5.26}\\
        \multicolumn{1}{l||}{} & \multicolumn{1}{l|}{FedTrimAvg}   & 12.6 & \textbf{6.03} & 6.5 & 5.91 & 6.21 & 6.05 & 12.39 & 7.33 & 7.46 & 13.04 & 6.82 & 5.43\\
        \multicolumn{1}{l||}{} & \multicolumn{1}{l|}{Krum}   & 35.77 & 19.28 & 10.24 & 51.03 & 18.99 & 11.68 & 16.23 & 9.71 & 7.41 & 17.73 & 10.35 & 7.07\\
        \hline
        \multicolumn{14}{c}{} \\
        \hline
        \multicolumn{1}{l||}{} & \multicolumn{1}{l|}{FedAvg}   & 13.1 & 6.49 & 7.38 & 6.09 & \textbf{5.69} & 5.8 & 13.17 & 8.69 & 9.44 & 12.73 & 6.99 & \textbf{6.24}\\
        \multicolumn{1}{l||}{} & \multicolumn{1}{l|}{FedAvgM}   & 12.85 & 6.37 & 7.09 & 5.84 & 5.84 & 6.79 & 12.74 & 8.7 & 8.56 & \textbf{12.31} & \textbf{6.3} & 7.11\\
        \multicolumn{1}{l||}{} & \multicolumn{1}{l|}{FedAdagrad}   & 19.54 & 8.58 & 6.82 & 12.98 & 6.55 & 7.18 & 14.52 & 9.43 & 7.83 & 13.64 & 6.33 & 7.32\\
        \multicolumn{1}{l||}{5} & \multicolumn{1}{l|}{FedAdam}   & 14.27 & 13.52 & 11.65 & 9.55 & 9.82 & 11.27 & 13.96 & 12.99 & 13.37 & 13.6 & 10.48 & 9.05\\
        \multicolumn{1}{l||}{} & \multicolumn{1}{l|}{FedYogi}   & \textbf{9.7} & 6.83 & \textbf{6.2} & \textbf{5.58} & 6.94 & \textbf{4.69} & 12.49 & \textbf{6.04} & \textbf{5.18} & 13.71 & 7.27 & 5.64\\
        \multicolumn{1}{l||}{} & \multicolumn{1}{l|}{FedMedian}   & 12.47 & \textbf{5.91} & 6.88 & 6.35 & 5.83 & 6.82 & \textbf{12.1} & 7.67 & 9.18 & \textbf{12.31} & 6.67 & 7.23\\
        \multicolumn{1}{l||}{} & \multicolumn{1}{l|}{FedTrimAvg}   & 13.32 & 6.56 & 6.71 & 6.32 & 5.95 & 6.58 & 12.97 & 8.68 & 9.36 & 12.37 & 6.5 & 6.78\\
        \multicolumn{1}{l||}{} & \multicolumn{1}{l|}{Krum}   & 33.57 & 18.88 & 10.83 & 38.61 & 15.48 & 9.04 & 16.8 & 10.82 & 10.87 & 17.4 & 10.76 & 8.37\\
        \hline
        \multicolumn{14}{c}{} \\
        \hline
        \multicolumn{1}{l||}{} & \multicolumn{1}{l|}{FedAvg}   & 13.91 & \textbf{6.29} & \textbf{7.48} & 7.29 & \textbf{5.18} & 7.13 & 13.49 & 10.47 & 10.39 & 12.31 & 7.21 & 8.01\\
        \multicolumn{1}{l||}{} & \multicolumn{1}{l|}{FedAvgM}   & 13.32 & 6.44 & 7.82 & 7.61 & 5.63 & \textbf{6.28} & 13.0 & 9.31 & 10.62 & 12.48 & 7.05 & 8.93\\
        \multicolumn{1}{l||}{} & \multicolumn{1}{l|}{FedAdagrad}   & 19.28 & 8.15 & 8.22 & 13.31 & 8.32 & 6.84 & 12.67 & \textbf{8.1} & \textbf{7.26} & 12.48 & 8.01 & 8.31\\
        \multicolumn{1}{l||}{10} & \multicolumn{1}{l|}{FedAdam}   & 14.77 & 12.65 & 8.66 & 11.65 & 7.39 & 8.21 & 14.86 & 11.35 & 11.26 & 13.86 & 12.45 & 11.34\\
        \multicolumn{1}{l||}{} & \multicolumn{1}{l|}{FedYogi}   & \textbf{10.85} & 8.64 & 8.09 & 7.74 & 6.54 & 8.38 & 15.41 & 12.58 & 9.18 & 12.26 & 9.63 & 8.99\\
        \multicolumn{1}{l||}{} & \multicolumn{1}{l|}{FedMedian}   & 12.44 & 7.02 & 9.64 & 6.77 & 6.56 & 8.25 & \textbf{12.36} & 9.8 & 10.6 & 12.42 & 8.65 & 8.28\\
        \multicolumn{1}{l||}{} & \multicolumn{1}{l|}{FedTrimAvg}   & 13.31 & 6.8 & 7.91 & \textbf{6.45} & 5.52 & 6.54 & 12.81 & 9.54 & 10.39 & \textbf{12.11} & \textbf{6.77} & \textbf{7.5}\\
        \multicolumn{1}{l||}{} & \multicolumn{1}{l|}{Krum}   & 29.6 & 18.43 & 13.73 & 30.62 & 15.27 & 10.45 & 16.98 & 13.64 & 12.93 & 17.2 & 12.57 & 10.61\\
        \hline
        \multicolumn{14}{c}{} \\

       \end{tabular}
        \caption{Average Chebychev, or $L_\infty$, distance to the size-based ground truth, lower values are better. The entries represent the mean worst client percentage difference from the ground truth; \textbf{Bold} entries are the best in the column.}
       \label{table:abss}
    \end{center}
\end{table*}

\section{Contribution instabilities between aggregation strategies}\label{instability}

Robustness, fairness, and generalization as parts of trustworthiness have become an important aspect of Artificial Intelligence \cite{li2022trustworthy}. This is no different in Federated Learning scenarios, where fairness of the global model and trust between participants are crucial to the operation of the federation. 

\begin{figure}[ht]%{0.4 cm}
   \centering
    \includegraphics[scale = .45]{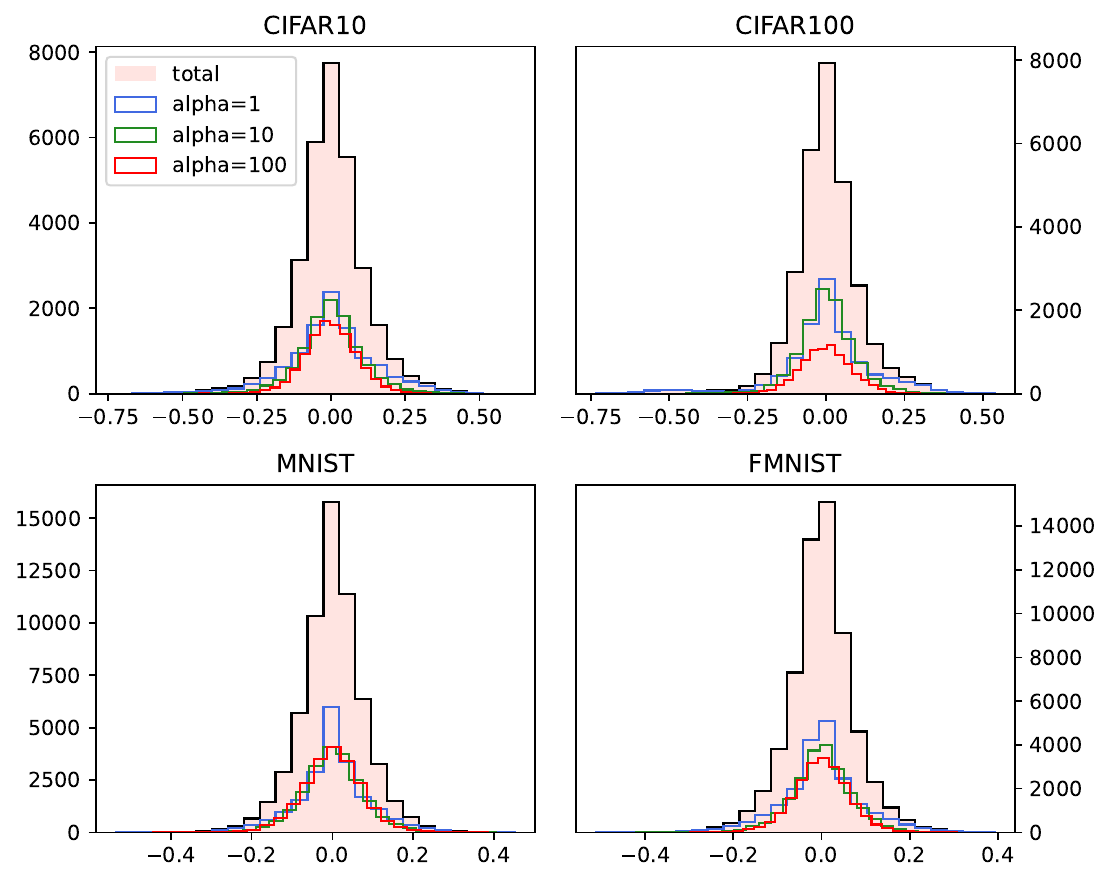}
   \caption{The distributions of client-wise contribution differences between strategies across all experiments. A value of 0.1 represents a client receiving 10\% more of the total reward using one strategy instead of another.}
   \label{fig:agg_diffs}
\end{figure}

In light of this, the previous findings introduce the question of the stability of reward payouts, in general. We have shown that similarity to a ground truth does not imply that client-wise performance is adequate. This leads to the question of how stable the contributions are when there is no ground truth but only different strategies which share an objective function. Indeed, there is no guarantee that the Shapley values display similar distributions across different aggregation strategies, even if they perform equally well with respect to a common ground truth. Assuming a shared environment, two strategies should produce the same contribution values: If neither the data distribution, nor the model or the objective function change, but only the server-side aggregation of models, there is no reason for contribution to behave any differently. This is essential in an industrial setting: If clients know that diverse strategies produce the same global model, but contributions differ, agreeing on the choice of an aggregation strategy seems impossible. Therefore, we analyze and discuss the stability of contributions, not with respect to the ground truth, but among aggregation strategies.

Our experiments show that there are severe discrepancies in contribution allocation between strategies. Figure \ref{fig:agg_diffs} shows the distribution of the differences in the client-specific contribution in all experiments. We observe that clients may receive rewards that differ greatly depending on the chosen aggregation strategy. Keeping in mind that the upper row contains the results of \textsc{Cifar-10} and \textsc{Cifar-100}, both using 3 clients, some examples contain discrepancies that reach $50 \%$ of the total reward. Naturally, the average difference is zero: One client's loss is another client's gain.

In addition, the figure contains $\alpha$-specific histograms. We observe that there is no value that performs significantly better than the others; in fact, the error distribution is stable between different values. This illustrates that no matter how heterogeneous the data, the problem persists. This is undesirable in a nonmalicious federation: A client receiving almost half of the total reward more depending on the aggregation technique is not in the interest of anyone involved.

\begin{figure}[ht]%{0.4 cm}
   \centering
    \includegraphics[scale = .4]{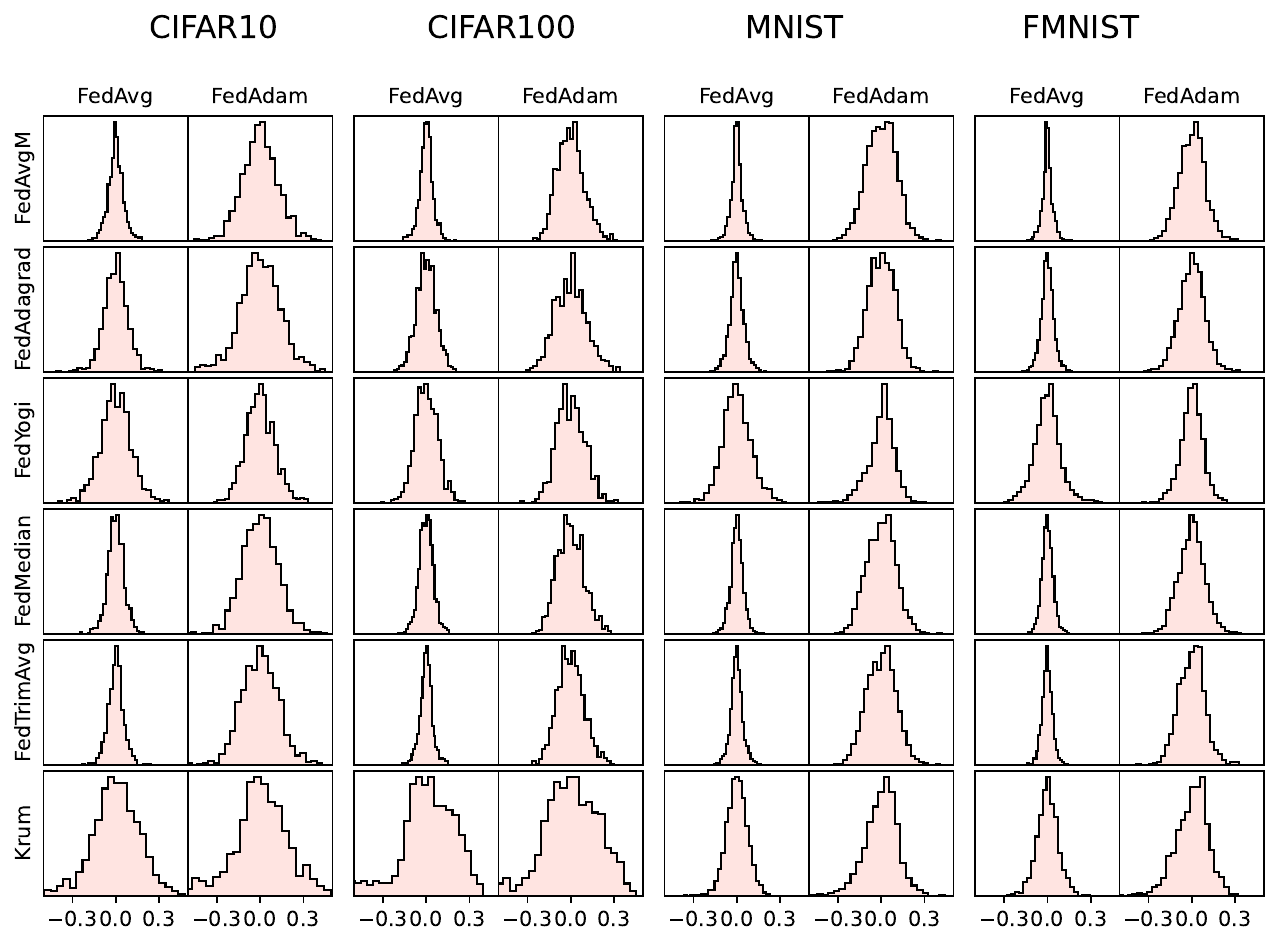}
   \caption{Histograms show contribution differences between pairs of strategies. X-axes are fixed across all plots.}
   \label{fig:agg_diffs_full}
\end{figure}

The scale of the problem is demonstrated in Figure \ref{fig:agg_diffs_full}, which shows the instability between different strategies. Ideally, all histograms would show no variance, with all differences close to zero. We show that this is not the case: Although the variance between strategies changes, there always exist substantially different contribution allocations. Although there are clearly some strategies that agree with others in most contributions, especially \textsc{FedAvg}, we show that there are still discrepancies of up to $30\%$ of the total reward when compared to other aggregations. In summary, we have shown that, even if the overall contribution lies close to the ground truth, different aggregation strategies will lead to differing contributions per client. As described above, this can cause substantial problems when establishing federations.

%\subsection{Implications}
\section{Discussion}
To the best of our knowledge, this work is the first to conduct a large-scale empirical analysis with findings that highlight a critical issue in federated learning: the inherent instability in the estimation of the contribution based on the per-round Shapley value across different aggregation strategies. This instability raises questions about the viability of current methods for determining client contributions in real-world federated learning environments, particularly in a cross-silo setting, where economic incentives are more common. One of the main implications of these results is the potential erosion of trust in the contribution evaluation methods and therefore in the federation at large. Without assurances in place, a participant whose contribution turns out to be far less than expected could argue that the chosen strategy disadvantages them and leaves the federation. This could hurt all participants, as the leaver might have possessed unique data from which the global model, and thus everyone involved, would profit. Trust is a cornerstone of federated learning, and its erosion could significantly impede its adoption in commercial applications.

Hence, we argue that although most aggregation strategies achieve comparable performance in terms of model accuracy, contribution evaluation using round-based Shapley values in Federated Learning lacks robustness, generalization, and fairness. By shedding light on this instability of contribution evaluation between some of the most popular strategies, we hope to provide a foundation for future aggregation strategies to self-evaluate not only their performance under traditional metrics but also the contribution evaluation stability.

Therefore, when configuring federation parameters in a real-world setting, it is crucial to ensure that any transition to an alternative aggregation strategy is accompanied by a comprehensive stability study such as the one presented in this work.  Through such a study, federations can ensure that the new aggregation strategy not only performs well in terms of model accuracy but also provides contribution stability, fostering a sustainable and trustworthy continuation of the federation.

In a more general context, we advocate for a reorientation in the design and evaluation of federated learning contribution systems. Beyond traditional metrics like accuracy, future research and development efforts should include the fairness and robustness of contributions in their studies to foster a sustainable and trustworthy federated learning ecosystem.

\section{Conclusions \& Future Work}\label{sec:conclusions}

%\textbf{Observations:} 
This work has brought to light a concern regarding the stability of popular Federated Learning strategies when evaluating the contribution of participants. Our findings underscore the challenges and risks associated with deploying cross-silo federated learning platforms in industrial settings, given the discrepancy of contribution amongst these strategies. Future proposed FL strategies should take our findings into account, conducting a performative analysis not only on the similarity of the strategies' contributions compared to a ground truth, but also on their stability compared to other strategies.

%\textbf{Implications and usefulness:} 
In a deployed cross-silo federation, a major task is to reduce friction between participants: Unlike what was observed, where a specific aggregation strategy might have been more beneficial to a client, statistically relevant contribution values may be sampled ahead. This would leaves the option to choose the strategy which is most beneficial to the federation as a whole, eliminating inter-client concurrence. This would lead to participants being rewarded more fairly, greatly improving trust in the remuneration process.

A potential venue of research would be to design an aggregation strategy based on an ensemble of different aggregation strategies designed to reduce the variance of the contribution allocation. This could mitigate the observed instability and maintain acceptable performance.

Finally, even if mitigation strategies are implemented for the observed instability, the computational complexity of Shapley value-based methods presents another hurdle in scaling these approaches to federations with a large number of participants. The computational effort of Shapley values scales exponentially with the number of clients, rendering simulations in high-client environments highly demanding. As federated learning gains traction in industrial and societal applications, such as healthcare and finance, the need for scalable, computationally efficient contribution evaluation techniques becomes paramount. Methods like \textsc{FastShap} \cite{jethani2022fastshap} or Monte Carlo sampling offer promising avenues for reducing the computational burden, but their impact on stability and fairness requires thorough investigation.

%%%%%%%%% REFERENCES
\bibliographystyle{IEEEtran}
\bibliography{main}

\end{document}